\documentclass[12pt,A4]{article}
\usepackage[margin=1in,footskip=0.25in]{geometry}

\usepackage{graphicx, color}  
\usepackage{amsmath,amssymb, amsfonts, MnSymbol, mathrsfs}
\usepackage{natbib}
\usepackage{hyperref}
\usepackage{authblk}     

\usepackage{amsthm}
\usepackage{geometry}
\usepackage{hyperref}
\usepackage{mathtools}
\usepackage{color}
\usepackage{xcolor}
\definecolor{indiagreen}{rgb}{0.07, 0.53, 0.03}

\usepackage{amsfonts,epsfig,epstopdf,url,array}
\usepackage{float}
\usepackage{pgfplots}
\usetikzlibrary {shapes.multipart}
\usepackage[normalem]{ulem}

\usepackage{algpseudocode}
\usepackage[titlenumbered,ruled,vlined]{algorithm2e}

\makeatletter
\renewcommand{\Indentp}[1]{%
  \advance\leftskip by #1
  \advance\skiptext by -#1
  \advance\skiprule by #1}%
\renewcommand{\Indp}{\algocf@adjustskipindent\Indentp{\algoskipindent}}
\renewcommand{\Indm}{\algocf@adjustskipindent\Indentp{-\algoskipindent}}
\makeatother

\usepackage{hyperref}
\usepackage{subcaption}
\usepackage{multirow}
\usepackage{multicol}

\newcommand{\R}{\mathbb{R}}
\newcommand{\N}{\mathbb{N}}

\newcommand{\X}{\mathcal{X}}

\newcommand{\x}{\boldsymbol{x}}

\newcommand{\E}{\mathbb{E}}

\newtheorem{thm}{Theorem}

\begin{document}
\title{Less Discriminatory Alternative and Interpretable XGBoost Framework for Binary Classification}
\author[1]{Andrew Pangia}
\author[1,2]{Agus Sudjianto}
\author[3]{Aijun Zhang\thanks{The views expressed in this paper are those of the authors and do not necessarily reflect those of Wells Fargo or H2O.ai.}}
\author[1]{Taufiquar Khan\thanks{Correspondence: aijun.zhang@wellsfargo.com, taufiquar.khan@charlotte.edu}}
\affil[1]{UNC Charlotte}
\affil[2]{H2O.ai}
\affil[3]{Wells Fargo}
\renewcommand\Affilfont{\itshape\small}
\date{} 
\maketitle
 
\begin{abstract}%
Fair lending practices and model interpretability are crucial concerns in the financial industry, especially given the increasing use of complex machine learning models. In response to the Consumer Financial Protection Bureau's (CFPB) requirement to protect consumers against unlawful discrimination, we introduce LDA-XGB1, a novel less discriminatory alternative (LDA) machine learning model for fair and interpretable binary classification. 
LDA-XGB1 is developed through biobjective optimization that balances accuracy and fairness, with both objectives formulated using binning and information value. It leverages the predictive power and computational efficiency of XGBoost while ensuring inherent model interpretability, including the enforcement of monotonic constraints.
We evaluate LDA-XGB1 on two datasets: SimuCredit, a simulated credit approval dataset, and COMPAS, a real-world recidivism prediction dataset. Our results demonstrate that LDA-XGB1 achieves an effective balance between predictive accuracy, fairness, and interpretability, often outperforming traditional fair lending models. This approach equips financial institutions with a powerful tool to meet regulatory requirements for fair lending while maintaining the advantages of advanced machine learning techniques.
\end{abstract}

\section{Introduction}
The financial industry increasingly relies on complex machine learning models for critical decision-making processes, such as credit scoring, loan approvals, and risk assessments. Though these models substantially improve predictive accuracy and efficiency, they also bring two major concerns: fairness and interpretability.

To address these concerns, the Consumer Financial Protection Bureau (CFPB) has enacted stringent regulations that require lenders to uphold fair lending practices and provide clear explanations for their decisions. On the fairness front, the CFPB mandates that lenders actively search for and adopt less discriminatory alternatives (LDAs). This ensures that models do not unfairly disadvantage protected groups in decisions like credit approval. Simultaneously, the adverse action reason code (AARC) requirement emphasizes model interpretability. Lenders must provide clear, specific, and understandable reasons for rejecting loan applications. This is crucial not only for transparency but also to meet legal obligations and maintain customer trust. 

The dual mandates of fairness and interpretability present a formidable challenge for financial institutions. Traditionally, machine learning models have been favored for their high accuracy and efficiency in solving binary classification problems. However, their use in regulated industries has raised concerns due to the complexity and potential for bias in decision-making. As a result, there is a pressing need for new approaches that strike a balance between accuracy, fairness, and interpretability.

In response to this challenge, we propose LDA-XGB1, a novel framework designed to meet both fairness and interpretability requirements while maintaining strong predictive performance. LDA-XGB1 builds upon the widely used XGBoost algorithm but ensures inherent interpretability by using shallow trees of depth 1, making each feature’s contribution to the decision easily understandable. To further enhance interpretability, monotonic constraints can be imposed on key features, ensuring logical and transparent behavior in predictions. LDA-XGB1 employs biobjective optimization that balances accuracy and fairness, incorporating XGB1-based binning techniques and information value to measure and mitigate potential disparate impact. Disparate impact occurs when a seemingly neutral model disproportionately affects a protected group (e.g., race, gender) without discriminatory intent. This ensures equitable treatment for protected groups in the decision-making process. As a result, the proposed LDA-XGB1 model offers a comprehensive solution that balances accuracy, fairness, and interpretability, making it highly applicable for compliance with regulatory requirements.

The rest of this paper is structured as follows. In Section 2, we provide a detailed review of related work, focusing on fairness-enhancing methods and interpretability techniques in machine learning. Section 3 outlines the methodological framework of LDA-XGB1, including its biobjective optimization and fairness metrics. In Section 4, we describe the datasets used for evaluation and present experimental results demonstrating the model's effectiveness in balancing accuracy, fairness, and interpretability. Finally, Section 5 discusses the implications of our findings for financial institutions and concludes with future directions for research.

\section{Background and Related Work}
\subsection{Interpretability in Machine Learning} 
In high-stakes domains like finance, model interpretability is essential for building trust, ensuring compliance with regulatory standards, and comprehending model behavior. With advancements in machine learning, two primary approaches have emerged: post-hoc explainability techniques and inherently interpretable models.

Post-hoc techniques aim to explain the behavior of complex, often black-box, models after they have been trained. Popular methods such as LIME (Locally Interpretable Model-Agnostic Explanations) \citep{ribeiro2016lime} and SHAP (SHapley Additive exPlanations) \citep{lundberg2017shap}, provide model-agnostic local explanations by approximating complex models with simpler surrogate models. However, studies have shown that post-hoc explanations may not always faithfully reflect the decision-making process of the underlying model, leading to misleading or incomplete interpretations \citep{kumar2020prob, bilodeau2024impossibility}. Moreover, the added layer of complexity in post-hoc analysis can be difficult to justify under regulatory scrutiny, especially in financial services where transparency is critical. These explanations are approximations, and there is no guarantee that they faithfully capture the complex interactions present in the black-box model. See also CFPB Circular 2022-03 \citep{CFPB2022-03} for a footnote comment that creditors must ensure the accuracy of any post-hoc explanations, as such approximations may not be viable with less interpretable models. 

In contrast to post-hoc methods, inherently interpretable models are designed to be transparent by nature \citep{sudjianto2021design}. These models ensure that their structure and behavior can be directly understood by users. One of the most prominent frameworks in this regard is the generalized additive model (GAM), which was first introduced by \cite{hastie1990} and later re-introduced for interpretable machine learning by \citep{vaughan1018xnn, yang2020exnn, agarwal2021nam}. GAMs provide a balance between flexibility and interpretability by assuming the input-output $(\x,y)$ relationship to be of the form
\begin{equation}
\sigma(\E(y|\x))  = g_0 + \sum_{j=1}^p g_j(x_j), 
\end{equation}
where $\sigma$ is the link function, $g_0$ is the overall mean and $g_{j}(x_j)$'s are the main effects of individual features. This additive structure allows each feature's contribution to the prediction to be independently analyzed and understood, making GAM models inherently interpretable and easy to visualize. Recent extensions of GAMs include GAMI-Net \citep{yang2021gami} with structured pairwise interactions based on neural networks and GAMI-Lin-Tree \citep{hu2023fanova} based on model-based trees. 

In summary, GAM and its extensions represent an ideal model structure for interpretable machine learning in regulated industries. Unlike post-hoc methods that attempt to explain complex models, GAMs are inherently interpretable by design, making them highly suitable for scenarios where model transparency and simplicity are paramount. 

\subsection{Fairness in Machine Learning}\label{sec:fairML}
Fairness has become a critical concern in the application of machine learning models, particularly in sectors such as finance, healthcare, and criminal justice, where decisions can significantly impact individuals' lives. In machine learning, fairness refers to the principle that models should not disproportionately harm or disadvantage specific groups defined by protected attributes such as race, gender, or socioeconomic status. \cite{caton2024fairnss} provides a comprehensive overview of fairness in machine learning, exploring the various definitions, metrics, and methodologies for integrating fairness into model development. 

In the context of model fairness, protected groups and reference groups are defined as follows:
\begin{itemize}
\item {\sf Protected Groups:} These are groups of individuals who are safeguarded against discrimination and bias based on legally recognized characteristics, also known as protected attributes. These attributes often include race, gender, age, religion, disability, sexual orientation, and ethnicity. Fairness in machine learning aims to ensure that models do not disadvantage these protected groups in their predictions, recommendations, or outcomes. For example, when evaluating the fairness of a credit decision algorithm, women or certain racial minorities might constitute a protected group if the concern is that the algorithm could unfairly disadvantage them.
\item {\sf Reference Groups:} These are groups against which the protected groups are compared when assessing fairness. They typically represent the dominant or non-disadvantaged population in relation to a specific attribute. For example, in gender-based fairness assessments, men might serve as the reference group when evaluating fairness toward women. The reference group serves as a benchmark for fairness, helping to identify disparities in outcomes between different groups.
\end{itemize}

In fairness assessments, models are examined to ensure that protected groups receive fair and equitable outcomes compared to reference groups, helping to prevent biases that can arise from societal inequalities being reflected in algorithmic decisions. Techniques such as fairness metrics, bias mitigation strategies, and group-based audits are used to measure and address these disparities.
The concept of Less Discriminatory Alternatives (LDA) is crucial as it underscores the obligation of financial institutions to actively seek out and implement strategies that reduce discrimination against protected groups. This proactive approach aligns with the requirements established by the Consumer Financial Protection Bureau (CFPB), which mandates that lenders identify and adopt less discriminatory alternatives whenever feasible. \cite{gillis2024lda} discusses the search for LDA in fair lending, providing a method to audit models for less discriminatory alternatives. Their approach allows regulators, auditors, and lenders to systematically discover alternative models that reduce discrimination or refute the existence of such alternatives.  

One effective approach to ensure fairness in machine learning is the incorporation of fairness constraints directly into the model's design and training process \citep{donini2018fairnss, zafar2019fairnss}. These constraints help to align the model's predictions with ethical considerations and regulatory requirements, ensuring that protected groups are treated equitably. 
In model development, fairness constraints can be mathematically formulated to minimize disparate impacts across different protected groups. For instance, one might impose constraints on the false positive rates or true positive rates to ensure that these rates are comparable across groups. By integrating these fairness constraints into the optimization process, models can be trained to balance predictive accuracy with fairness, effectively reducing bias in decision-making. Note that such fairness optimization problems can often be equivalently formulated as biobjective or multi-objective programming (MOP) tasks also known as $\epsilon$-constraint scalarization \citep{haimes1971mop}, where 
fairness objectives are explicitly considered alongside performance objectives. This approach enables model developers to explore trade-offs and find solutions that provide both high accuracy and improved fairness.

In summary, biobjective or multi-criteria optimization that integrates fairness objectives has proven to be an effective approach in developing LDA for machine learning models, particularly in the domain of fair lending. 

\section{Methodology}
In this section, we outline the development of LDA-XGB1, a novel framework designed to balance fairness and accuracy while ensuring model interpretability. LDA-XGB1 builds upon the widely-used XGBoost algorithm, but it incorporates several enhancements to address the dual objectives of fairness and interpretability. The model employs biobjective programming that balances predictive performance and fairness, with built-in mechanisms to ensure that decisions remain transparent and understandable. 

\subsection{XGB1 Base Model}
XGB1 is an inherently interpretable model advocated by  PiML toolbox \citep{sudjianto2023piml} and it serves as the base model for the proposed method. XGB1 is built upon the widely used XGBoost framework \citep{chen2016xgboost}, a powerful machine learning algorithm that utilizes gradient-boosted decision trees. XGBoost (Extreme Gradient Boosting) is known for its high accuracy, speed, and efficiency, making it a popular choice for various predictive tasks, including classification and regression.

The XGB1 base model employs depth-1 trees (also known as decision stumps) during the boosting process. A key characteristic of this model is its interpretability, as it can be viewed as a Generalized Additive Model (GAM) with piecewise constant main effects. The following procedure outlines how XGB1 achieves this interpretability:
\begin{itemize}
\item {\sf Collecting Unique Splits for Each Feature:}
In each boosting step, XGB1 adaptively selects a feature and determines the best split point. For a given feature, the collected set of unique splits (across different boosting steps) would divide the feature into discrete intervals (bins), with each interval representing a specific range of values.
\item {\sf Calculating Accumulated Leaf Node Values:}
After determining the split points for each feature, the accumulated leaf node values for each bin generated by the unique splits are calculated. These values correspond to the contribution of each bin toward the final prediction.
\end{itemize}

The key feature of XGB1 is its ability to adaptively select split points, which enables the model to capture feature behavior in a data-driven manner. This adaptive selection allows XGB1 to achieve superior predictive performance compared to spline-based GAMs, which rely on pre-defined smooth functions for feature effects. 

\medskip
\noindent {\bf Support for Regularization:} 
To prevent overfitting and enhance generalization ability, XGBoost incorporates built-in regularization mechanisms, including L1 (lasso) and L2 (ridge) regularization. These regularization techniques penalize large coefficients in the model, ensuring a balance between complexity and performance.  

\medskip
\noindent {\bf Support for Monotonic Constraints:} 
XGBoost supports the implementation of monotonic constraints, which ensures that the relationship between specific features and the model's predictions remains consistent with domain knowledge. For example, in financial applications, it is expected that an increase in income should not lead to a decrease in the likelihood of loan approval. By imposing these constraints during the training process, XGB1 guarantees that predictions follow logically consistent patterns.

\subsection{Binning and Information Value}\label{sec:IV}
Binning and Information Value (IV) are fundamental tools in assessing the predictive strength of explanatory variables in models, especially in binary classification tasks such as credit scoring or risk modeling. These methods help us evaluate how well input variables (features) distinguish between two outcome classes (e.g., default vs. non-default). Importantly, they can also be adapted to promote fairness by minimizing the influence of variables that may cause disparate impacts to protected groups.

\paragraph{Binning Process.} 
Binning refers to dividing continuous variables into discrete intervals, or ``bins", which group similar values together. This simplification helps reduce noise while preserving the variable’s predictive power. We begin the binning process by using XGB1 base model. These initial bins serve as a starting point for understanding the distribution of values.

To further refine the binning process, we utilize \emph{optbinning}, an optimization technique based on mixed-integer programming. This method identifies the best binning strategy by maximizing a measure known as Jeffreys’ Divergence, which reflects how well the variable differentiates between two classes. In predictive modeling, this divergence is commonly represented as IV.

\paragraph{IV Calculation.}
IV measures the strength of a variable’s ability to distinguish between different outcomes. It is calculated by comparing the proportion of ``events'' (e.g., defaults) and ``non-events'' (e.g., non-defaults) within each bin. The formula for IV is:
\begin{align}\label{IV}
IV=\sum_{i=1}^n(p_i-q_i)\log\left(\frac{p_i}{q_i}\right), 
\end{align}
where
\[p_i=\dfrac{r^{E}_{i}}{r^{E}_T}, \quad q_i=\dfrac{r^{NE}_{i}}{r^{NE}_T},\] and
\begin{itemize}
    \item $r^E_i$ is the number of events in bin $i$
    \item $r^{NE}_i$ is the number of non-events in bin $i$
    \item $r^E_T$ is the total number of events
    \item $r^{NE}_T$ is the total number of non-events.
\end{itemize}
This formula captures the difference between the event and non-event distributions across the bins, with larger differences contributing more to IV. Higher IV indicates that the variable is better at separating the two outcomes, making it more predictive.

\paragraph{Interpreting IV.}
The general rule for interpreting IV \citep{Mironchyk2017, siddiqi2006credit} is
\begin{enumerate}
    \item $IV<0.02$: the variable has little to no predictive power.
    \item $0.02\leq IV <0.1$: the variable provides weak predictive power.
    \item $0.1\leq IV <0.3$: the variable has moderate predictive power.
    \item $IV\geq 0.3$: the variable is highly predictive.
\end{enumerate}
This ranking helps us determine which variables are most useful in a model, allowing us to focus on features that provide the greatest predictive value.

\paragraph{Ensuring Fairness -- Minimizing IV.} 
While maximizing IV for binary outcomes is essential for improving model accuracy, fairness considerations require a different approach for certain variables that may cause disparate impacts. As discussed in Section~\ref{sec:fairML}, fairness means that the model should eliminate or reduce the influence of predictive variables with disparate impacts to protected groups. 
%
In this context, our goal is to minimize the IV for such variables with respect to their ability to discriminate between reference and protected groups in such a way they don’t unduly influence model predictions.

Thus, it creates a balance between
\begin{itemize}
    \item Maximizing the IV for explanatory variables to improve accuracy,
    \item Minimizing the IV for explanatory variables with disparate impacts to reduce bias and promote fairness.
\end{itemize}
By minimizing the IV of variables with disparate impacts, we ensure that they play a minimal role in model decision-making, promoting fairness and reducing potential bias in the model.

The \emph{optbinning} algorithm \citep{navas2020optimal} helps us achieve this balance by optimizing the binning process for each explanatory variable. The goal is to merge or adjust bins in ways that increase IV for binary outcomes, for improving  the model’s predictive ability. However, for fairness, the strategy is the opposite: the binning is optimized to reduce IV of variables with disparate impacts. 

This dual-objective binning process helps ensure that models are both accurate and fair. For example, a variable like debt to income ratio might have its bins adjusted to increase IV and better predict risk, while a variable like mortgage size would have its bins merged or adjusted to reduce IV, preventing it from having a large influence on the predictions.


In summary, binning and IV are not just tools for improving model accuracy, they also play a crucial role in ensuring fairness. By carefully managing the IV's for both binary outcomes and protected groups, we can build models that are both effective and equitable.

\subsubsection{Accuracy Metric}
Consider a set of $n\in\N$ bins. For any two specific bins $i$ and $j$, let $X_{ij}\in\{0,1\}$ such that $X_{ij}=1$ indicates that bins $i$ through $j$ are merged, while $X_{ij}=0$ indicates otherwise. Refer to \citep{navas2020optimal} for the commonsense constraints on $X_{ij}$ such that noncontiguous bins cannot be merged without all the bins between them, and a single pre-bin cannot be merged into multiple superbins.
With the definition of such binary variables $X_{ij}$, the IV becomes the following function of $X$:
\begin{align}
    IV:X\mapsto \sum_{i=1}^n\left(\left(\sum_{j=1}^i(p_j-q_j)X_{ij}\right)\log\left(\dfrac{\sum_{j=1}^i p_jX_{ij}}{\sum_{j=1}^i q_jX_{ij}}\right)\right). \label{IV2}
\end{align}
Note in \eqref{IV2} that the interior of the summation is the IV over each merged bin, even though, for some integer $k< i$, $X_{ij}=0$ for all $j\leq k$.

We follow \cite{navas2020optimal} to define $V_{ij}$ as the IV for a single superbin containing the $j$th through $i$th bins for all $j\leq i$,
\begin{align}\label{IV3}
    V_{ij}:=\left(\sum_{z=j}^i(p_z-q_z)\right)\log\left(\dfrac{\sum_{z=j}^i p_z}{\sum_{z=j}^i q_z}\right). 
\end{align}
Then, construct a linear representation of $X$ as the accuracy metric for binary classification,
\begin{align}\label{objILP}
f_{\rm acc}(X):=\sum_{i=1}^n \left(V_{ii}X_{ii} + \sum_{j=1}^{i-1}(V_{ij}-V_{ij+1})X_{ij}\right).
\end{align}
Note that the second summation is required, in order to remove IVs for bins that have been merged. In Appendix we provide a thorough demonstration of this result as a theorem, along with a technical proof by induction.

\subsubsection{Fairness Metric}\label{sec:FairnessMetric}
Denote by $y_A$ the indicator of the reference group (0) or the protected group (1). 
For an explanatory variable used in model development, in the process of merging bins, this time we want to minimize the IV with respect to $y_A$. That is, the model should not be able to discriminate between reference and protected groups. In the similar fashion, define $r_T^{AE}$ to be the total number of events in the protected group (i.e., $y_A = 1$), $r_T^{NAE}$ to be the total number of events in the reference group (i.e., $y_A = 0$). For each bin $i$, let $r^{AE}_i$, $r^{ASE}_i$  be the number of protected events and the number of reference events in bin $i$, respectively. Then, write 
\begin{align}
    p^{A}_i=\dfrac{r^{AE}_{i}}{r^{AE}_T}, \quad q^{A}_i=\dfrac{r^{NAE}_{i}}{r^{NAE}_T}. \label{fairDists}
\end{align}
Similar to (\ref{IV3}), by expressing
\[
    V^{A}_{ij}:=\left(\sum_{z=j}^i(p^{A}_z-q^{A}_z)\right)\log\left(\dfrac{\sum_{z=j}^i p^{A}_z}{\sum_{z=j}^i q^{A}_z}\right),
\]
we can define the following fairness metric,
\begin{align}\label{fairObj}
f_{\rm fair}(X)=\sum_{i=1}^n \left(V^{A}_{ii}X_{ii} + \sum_{j=1}^{i-1}(V^{A}_{ij}-V^{A}_{ij+1})X_{ij}\right).
\end{align}
It is important to note that, while we maximize $f_{\rm acc}(X)$ to obtain as much information about the target as possible, we want to minimize 
$f_{\rm fair}(X)$ to ensure that the binning yields as little information as possible. 

\subsection{Biobjective Programming}
When balancing two objective functions -- accuracy and fairness, we utilize biobjective programming (BOP) techniques, which can be expressed as follows:
\begin{align}
\min_{X\in \X} \;\left[-f_{\rm acc}(X),\; f_{\rm fair}(X)\right] \tag{BOP}\label{bopBasic}
\end{align}
where $f_{\rm acc},f_{\rm fair}:\R^n\to\R$ are the accuracy and fairness metrics, and $\X\subset\{0,1\}^{n\times n}$ is the feasible set defined by a finite list of constraints. The solutions to \eqref{bopBasic} are the points $x$ which result in the best $f_{\rm acc}(x)$ or $f_{\rm fair}(x)$. To find these, we implement the $\epsilon$-constraint method \cite{haimes1971mop}.

The $\epsilon$-constraint method exploits the fact that the two objective functions conflict by incorporating one objective function as a constraint. Let $\epsilon\in f_{\rm fair}(\X)$; we choose to optimize $f_{\rm acc}$ with the fairness constraint:
\begin{align}
\max_{\x\in \X} &\;f_{\rm acc}(X) \nonumber\\
\text{S.t.} &\;f_{\rm fair}(X)\leq \epsilon \tag{$\epsilon$Con}\label{eSop}
\end{align}
Let $X^{1*}, X^{2*}\in\X$ minimize $-f_{\rm acc}(X)$ and $f_{\rm fair}(X)$ respectively over $\X$. By adjusting the FIV (fairness IV) bound
$$
\epsilon\in[f_{\rm fair}(X^{2*}), f_{\rm fair}(X^{1*})],
$$
we are able to obtain all solutions to \eqref{bopBasic}. Frequently, the solutions obtained are plotted against one another in what is called a Pareto front so as to demonstrate the trade-off. In practice, the choice of $\epsilon$ is typically selected based on some previously determined thresholds.

\subsection{The Proposed LDA-XGB1}
The proposed algorithm for constructing the model sequence and the approach to demonstrate interpretability for the resulting model are explained below. We begin by developing a binary classification ML model, with the option of applying increasing and decreasing monotonicity constraints as appropriate to aid interpretability, followed by construction of the feature importance and main effect  plots for model interpretation. 

\subsubsection{Optimization Algorithm}
Suppose we have a dataset consisting of samples $(\x, y, y_A)$ where $x\in \R^{k}$ are $k$-dimensional predictors, $y\in\{0,1\}$ is the target, and $y_A\in\{0,1\}$ is the protected group indicator. Partition the dataset into training and testing sets and begin by applying XGB1 base model \citep{sudjianto2023piml} to the training set. This provides us with sets of prebins for each predictor.

\begin{algorithm}[ht!] 
\SetNlSty{text}{Step }{:} 
\SetAlgoNlRelativeSize{0}
\caption{LDA-XGB1} \label{alg:fairSequence} 
\SetKwInOut{Input}{Input}
\SetKwInOut{Output}{Output}
\Input{Dataset $\{(\x,y,y_A)\}$, FIV bound $\bar\epsilon$}
\Output{BinLogistic}
{
    \nlset{$0$} Construct train/test split
    
    \nlset{$1$} Apply XGB1 to get prebins for each predictor $x_i$, $i\in\{1,\dots,k\}$ \label{sequenceXGB} \\
    
    \nlset{$2$} \For{$i\in\{1,\dots,k\}$}{Merge prebins for maximizing $f_{\rm acc}(X)$ by  \\
    \Indp \uIf{fairness constraint}
        {\color{black} Program \eqref{eSop} using $f_{\rm fair}(X)\leq  \bar\epsilon$}
    \uElse
        {Program \eqref{eSop} omitting $f_{\rm fair}(X)$ constraint 
        }
    }\label{sequenceIVOpt}

    \nlset{$3$} Fit a binning logistic regression model based on the optimally merged bins upon one-hot encoding \label{sequenceGLM} 
}
\end{algorithm}

Algorithm~\ref{alg:fairSequence} outlines the key algorithmic steps. Step 2 is to optimally merge prebins for each predictor by \eqref{eSop} with or without the fairness constraint according to the input FIV bound $\bar\epsilon$. Without the fairness constraint, it is identical to the default optbinning approach \citep{navas2020optimal}, which serves as the baseline for model comparison.  In Step 3, it fits a binning logistic regression model based on the optimized superbins upon the use of one-hot encoding. Finally, the LDA-XGB1 algorithm outputs such a BinLogistic model. 

Note that in both Steps 1 and 3, the monotone constraints (i.e., ascending/descending trends) can be imposed on some key variables in order to achieve proper interpretability. Step 1 takes the advantage of XGBoost framework in XGB1 model training with monotone constraints, while in Step 3 we use the {\sf optimize} function in Python Scipy to compute for a constrained logistic regression model. 

\subsubsection{Model Interpretation}
To interpret the resulting BinLogistic model from the LDA-XGB1 algorithm, we present another algorithm of model interpretability. Following \citep{yang2021gami},  we develop the procedures to generate the FI (feature importance) and ME (main-effect) plots. The ME plot visualizes the  univariate function for each predictor to display the partial input-output relationship. For LDA-XGB1 based on optimal superbins, each ME plot is a step function with piece-wise constant values.  The FI plot is presented by a bar chart, with each predictor's importance index quantified as the sample variance of the step function. See Algorithm~\ref{alg:FI} for the concrete steps. 

\begin{algorithm}[ht!] 
\SetNlSty{text}{Step }{:} 
\SetAlgoNlRelativeSize{0}
\caption{LDA-XGB1 Interpretability} 
\label{alg:FI} 
\SetKwInOut{Input}{Input}
\SetKwInOut{Output}{Output}
\Input{BinLogistic}
\Output{FI and ME plots}
{
    \nlset{$1$} Extract model coefficients
        
    \nlset{$2$} \For{each feature}{
    \For{each bin $E_{ik}$}{
    Assign the coefficient $c_{ik}$ as the piece-wise constant function value
    }
    Output the step function
    }
    \nlset{$3$} Draw the step function as the main effect plot for each feature\\
    \nlset{$4$} Calculate the sample variance of the step function as the importance index for each feature, then draw the feature importance plot
    }
\end{algorithm}

\section{Numerical Experiments}
In this section, we demonstrate the efficacy of the proposed algorithm using both simulated and real world datasets. We consider two datasets, a simulated credit dataset, SimuCredit, from \cite{sudjianto2023piml}, as well as the well-known, more complicated, COMPAS dataset, obtained from \cite{tempeh}. For each dataset, we present the LDA-XGB1 model from Algorithm \ref{alg:fairSequence}. We consider the main effect plots to observe the behaviour implied by this model, add monotone constraints to induce sensible behaviour, and follow up by considering fairness.

To confirm the performances of the resulting LDA-XGB1 models, we employ the widely used metrics, AUC (area under the ROC curve) to measure the accuracy, and  AIR (adverse impact ratio) to measure the fairness, with the latter defined as follows: 
\begin{equation}\label{AIR}
AIR=\dfrac{(TP_1+FP_1)}{(TP_0+FP_0)}.
\end{equation}
Note that the AIR  compares the number of times $y=1$ is accepted against the number of times $y=0$ is accepted. The goal is to have the AIR as close to 1 as possible. 

\subsection{SimuCredit}\label{sec:simu}
SimuCredit is a simulated credit approval dataset from the PiML toolbox \cite{sudjianto2023piml}. It contains the following nine features:
\begin{enumerate}\setlength\itemsep{0em}
    \item {\sf Mortgage:} The original amount of the mortgage payment
    \item {\sf Balance:} The amount left on the mortgage payment
    \item {\sf Amount Past Due:} the amount of money in missed payments
    \item {\sf Delinquency:} the amount of time spent with missed payments
    \item {\sf Inquiry:} The number of applications for credit which required companies to check the credit score of the datapoint
    \item {\sf Open Trade:} The number of accounts on which payments have to be made (credit cards, car payments, etc.)
    \item {\sf Utilization:} the ratio of balance to credit limit 
    \item {\sf Gender:} binary variable denoting male or female
    \item {\sf Race:} binary variable denoting black and white
\end{enumerate}
The target variable is {\sf Status}, with 1 being approved and 0 otherwise. In this experiment, we consider {\sf Race} as the protected group indicator. As a standard practice in fair lending,  we remove both {\sf Race} and {\sf Gender} in model development, so that there will be no disparate treatment (intentional discrimination). Even when the demographic variables are not used in the modeling process, the model can still have disparate impacts (unintentional discrimination) caused by other variables. Thus, there is need to search for less discriminatory alternatives.

We split the data into training and test sets, using SKlearn's `train\_test\_split' function with a constant random seed (23) for reproducibility. We employ Algorithm~\ref{alg:fairSequence} using the XGBoost.XGBClassifier method with max\_depth = 1, n\_estimators = 1000 , learning rate = 0.3, tree method = `auto', max\_bins = 256, reg\_lambda = 1, and reg\_alpha = 0. 

We run three rounds of the LDA-XGB1 algorithm with the following settings:
\begin{itemize}
\item[1)] {\sf Plain XGB1 Model.} 
We applied Algorithm~\ref{alg:fairSequence} without fairness or monotone constraints. It fits a plain XGB1 model with optbinning for merging prebins. The feature importance and main effect plots are shown in Figures~\ref{fig:simuNoMonoFI} and \ref{fig:simuNoMonoME}, respectively. The main effect plots show apparent monotonic patterns for {\sf Utilization}, {\sf Balance} and {\sf Mortgage}. 

\begin{figure}[htp!]
    \centering
    \includegraphics[width=0.6\linewidth]{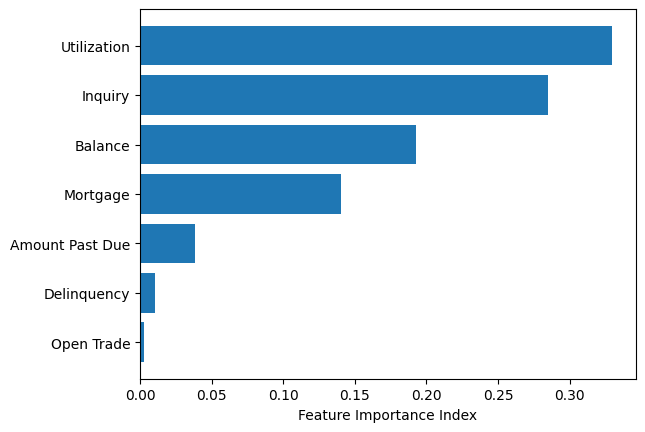}
    \caption{SimuCredit data: feature importance plot of the plain XGB1 model.}
    \label{fig:simuNoMonoFI}
\bigskip
    \centering
    \begin{subfigure}{0.4\textwidth}
    \includegraphics[width=\textwidth]{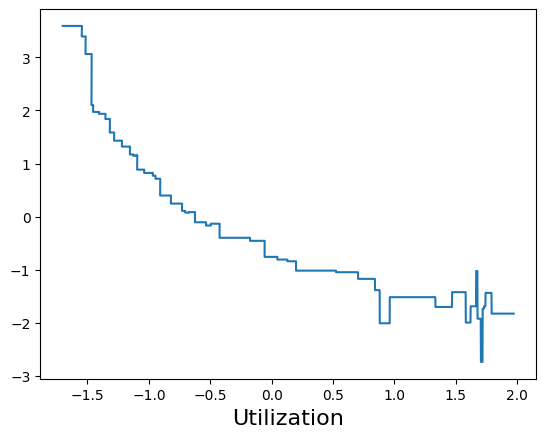}
    \caption{}
    \end{subfigure}
    \begin{subfigure}{0.4\textwidth}
    \includegraphics[width=\textwidth]{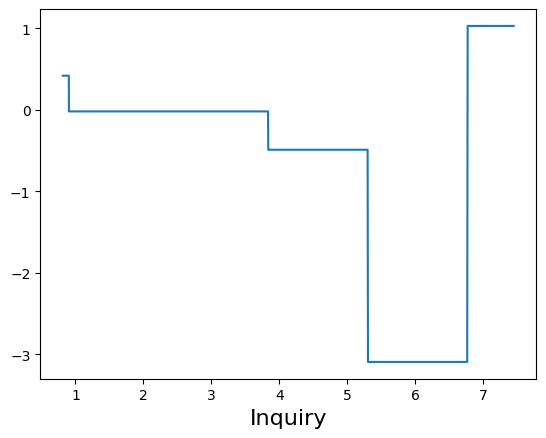}
    \caption{}
    \end{subfigure}
    \begin{subfigure}{0.4\textwidth}
    \includegraphics[width=\textwidth]{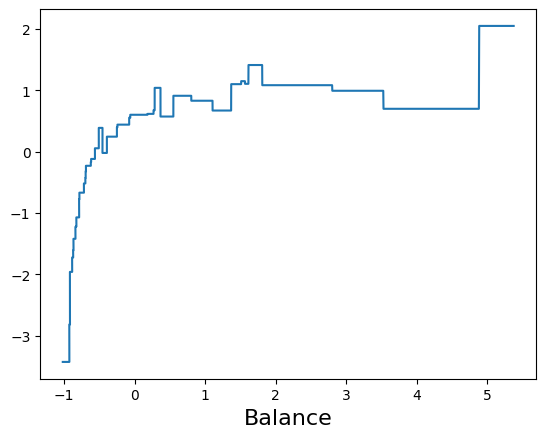}
    \caption{}
    \end{subfigure}
    \begin{subfigure}{0.4\textwidth}
    \includegraphics[width=\textwidth]{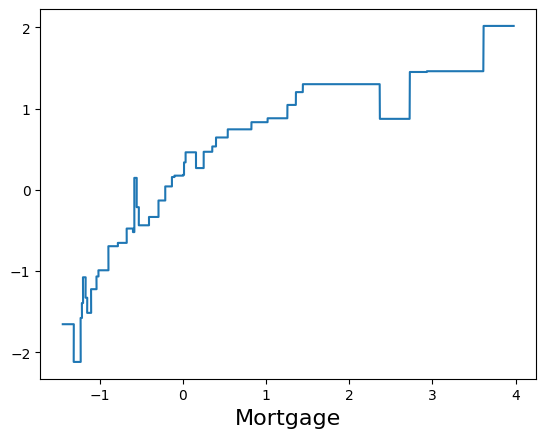}
    \caption{}
    \end{subfigure}
    \caption{SimuCredit data: main effect plots of top features by the plain XGB1 model.}
    \label{fig:simuNoMonoME}
\end{figure}

\begin{figure}[htp!]
    \centering
    \includegraphics[width=0.6\linewidth]{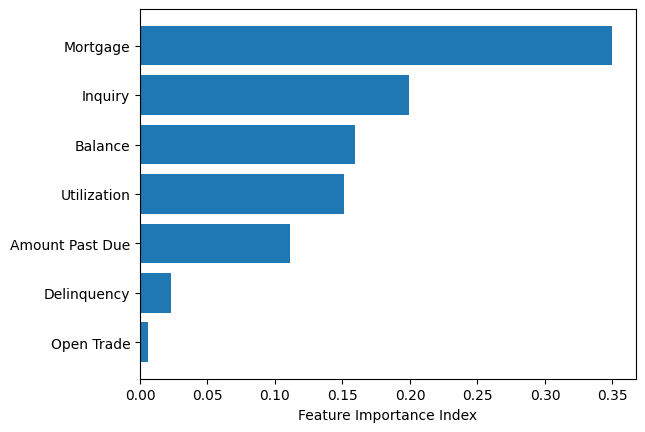}
    \caption{SimuCredit data: feature importance plot of the monotonic XGB1 Model. 
    }
    \label{fig:simuMonoFI}
\bigskip
    \centering
    \begin{subfigure}{0.4\textwidth}
    \includegraphics[width=\textwidth]{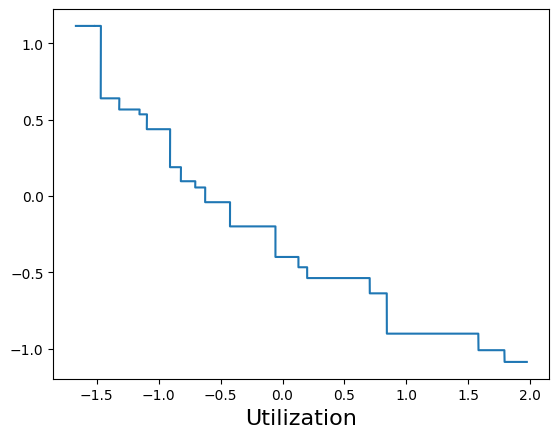}
    \caption{}
    \end{subfigure}
    \begin{subfigure}{0.4\textwidth}
    \includegraphics[width=\textwidth]{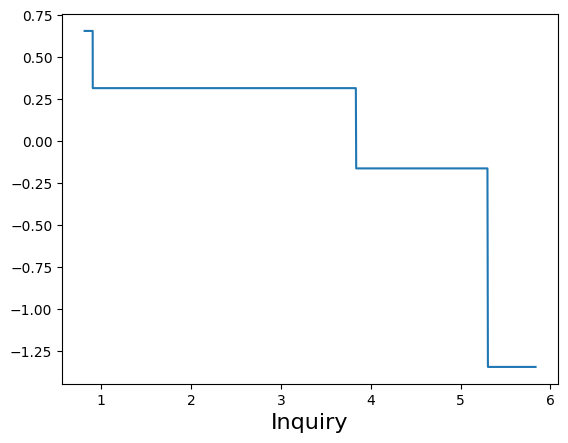}
    \caption{}
    \end{subfigure}
    \begin{subfigure}{0.4\textwidth}
    \includegraphics[width=\textwidth]{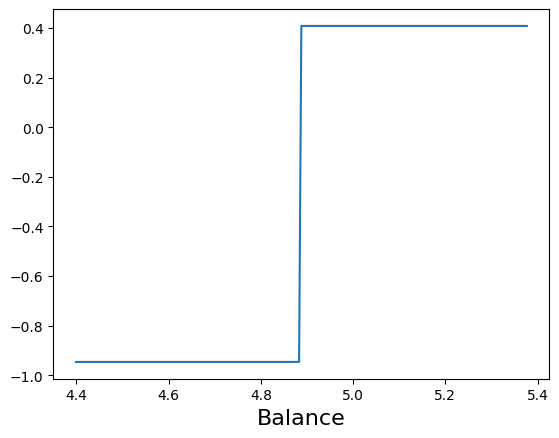}
    \caption{}
    \end{subfigure}
    \begin{subfigure}{0.4\textwidth}
    \includegraphics[width=\textwidth]{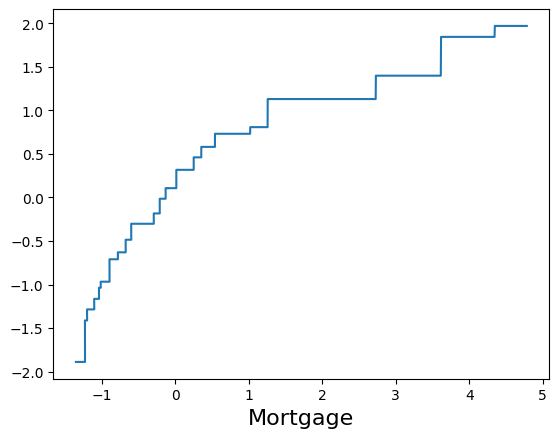}
    \caption{}
    \end{subfigure}
    \caption{SimuCredit data: Main Effect plots of the features, {\sf Utilization}, {\sf Inquiry}, {\sf Balance}, and {\sf Mortgage}, of the monotonic XGB1 model.}
    \label{fig:simuMonoME}
\end{figure}

\item[2)] {\sf Monotonic XGB1 Model.} 
As observed from the unconstrained XGB1 model, we impose monotone constraints such that {\sf Balance} and {\sf Mortgage} must be increasing, while {\sf Utilization}, {\sf Inquiry}, {\sf Amount Past Due}, {\sf Delinquency}, and {\sf Open Trade} must be decreasing. 
Refit the LDA-XGB1 model and we have the feature importance and main effect plot results in Figures \ref{fig:simuMonoFI} and \ref{fig:simuMonoME}, respectively. Note that {\sf Mortgage} becomes far more important with the introduction of monotonic constraints, while {\sf Utilization} loses importance to the point that it's now fourth important. The monotonic shapes of the main effect plots match our expected interpretability based on reality.

Figure~\ref{tab:simuPerf} tabulates the AIR and AUC results for both the plain XGB1 and monotonic XGB1 baseline models in the bottom rows. It is worth noting that for the training data, imposing monotonicity makes XGB1 perform better in AUC and worse in AIR, while for the testing data, the monotonic version has better performances in both AUC and AIR. When the monotonic constraints are imposed wisely, the resulting model can have both better generalization and better interpretability.  
Nevertheless, the AIR values are far from the normal acceptable range $[0.8, 1.2]$, which means both baseline models have disparate treatment to the protected `race' groups.

\begin{figure}
    \centering
    \begin{subfigure}{0.45\linewidth}
        \includegraphics[width=\linewidth]{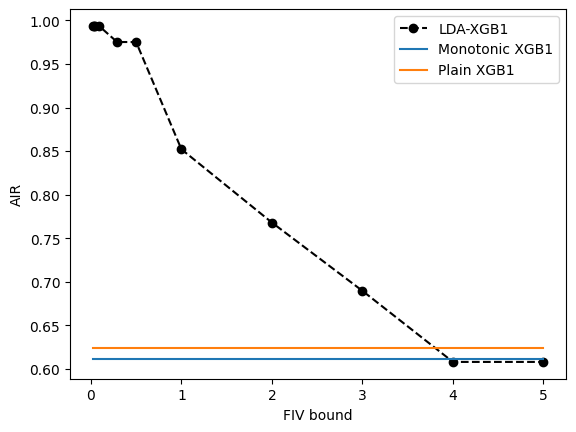}
        \caption{AIR}
        \label{fig:simuMonoAIR}
    \end{subfigure}
    \begin{subfigure}{0.45\linewidth}
        \includegraphics[width=\linewidth]{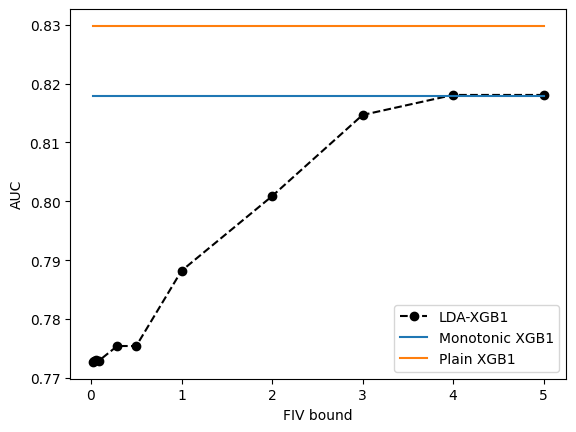}
        \caption{AUC}
        \label{fig:simuMonoFairAUC}
    \end{subfigure}
    \caption{SimuCredit data: LDA-XGB1 with fairness and accuracy trade-off.
    }\label{fig:simuPerf} 

\bigskip\bigskip
    \centering
\begin{tabular}{r|rr|rr}\hline
LDA-XGB1 & AIR Train & AUC Train & AIR Test & AUC Test \\ \hline
5.0 & 0.607805 & 0.820860 & 0.608045 & 0.818078 \\
4.0 & 0.607805 & 0.820860 & 0.608045 & 0.818078 \\
3.0 & 0.704422 & 0.815850 & 0.689610 & 0.814674 \\
2.0 & 0.771618 & 0.801231 & 0.767983 & 0.800871 \\
1.0 & 0.848237 & 0.787234 & 0.852124 & 0.788191 \\
0.5 & 0.984535 & 0.773657 & 0.975157 & 0.775408 \\
0.29 & 0.984535 & 0.773657 & 0.975157 & 0.775408 \\
0.09 & 0.998703 & 0.769939 & 0.993632 & 0.772940 \\
0.05 & 0.998703 & 0.769939 & 0.993632 & 0.772940 \\
0.03 & 0.998703 & 0.769939 & 0.993632 & 0.772940 \\
0.019 & 0.998269 & 0.769843 & 0.993632 & 0.772704 \\\hline
Plain XGB1 & 0.614619 & 0.821284 & 0.611299 & 0.817928 \\
Monotonic XGB1 & 0.609502 & 0.840745 & 0.624152 & 0.829808 \\ \hline
\end{tabular}

\caption{SimuCredit data: AIR and AUC for the LDA-XGB1 models with varying FIV bounds $\bar\epsilon\in\{0.019, 0.09, 0.29, 0.5, 1.0, 2.0, 3.0, 4.0, 5.0\}$, together with baseline XGB1 models without using fairness constraints.} \label{tab:simuPerf}
\end{figure}

\item[3)] {\sf Fair Monotonic XGB1 Model.}
Consider the protected/reference groups defined by `race', for which we define the fairness metric \eqref{fairObj}. We run the LGD-XGB1 algorithm with the fairness constraints and varying FIV bounds $\bar\epsilon\in\{0.019, 0.09, 0.29, 0.5, 1, 2, 3, 4, 5\}$. The trade-off between fairness and accuracy are shown in Figure~\ref{fig:simuPerf} for the testing performances, as well as the tabular results in Figure~\ref{tab:simuPerf} for both training and testing performances. It is evident that as FIV bound decreases, the model tends to fairer with higher AIR values, while the prediction performance gets sacrificed with lower AUC values. 
\end{itemize}

In search of LDA models with fairness constraints, suppose the resulting model is desirable to have AIR value at least 0.8. By checking Figure~\ref{fig:simuPerf} or Figure~\ref{tab:simuPerf}, we should choose the FIV bound $\bar\epsilon = 1.0$ with the corresponding model testing performances AIR 0.8521 and AUC 0.7882. Note that we can also run finer FIV bounds between $\bar\epsilon \in [1.0, 2.0]$ to search for a better trade-off. This final model at $\bar\epsilon = 1.0$ can be interpreted in the same way by feature importance and main effect plots, as shown in Figures~\ref{fig:simuMonoFairFI} and \ref{fig:simuMonoFairME}. It can be found that the predictor {\sf Mortgage} is substantially more important when considering fairness, and its binning effect gets coarsened in order to mitigate the bias between the protected and reference `race' groups.

\begin{figure}[htp!]
    \centering
    \includegraphics[width=0.6\linewidth]{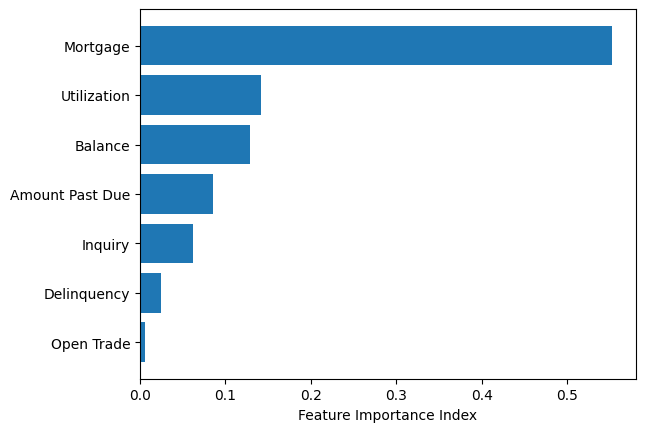}
    \caption{SimuCredit data: feature importance of LDA-XGB1 with FIV bound $\bar\epsilon=1.0$.}
    \label{fig:simuMonoFairFI}
\bigskip
    \centering
    \begin{subfigure}{0.4\textwidth}
    \includegraphics[width=\textwidth]{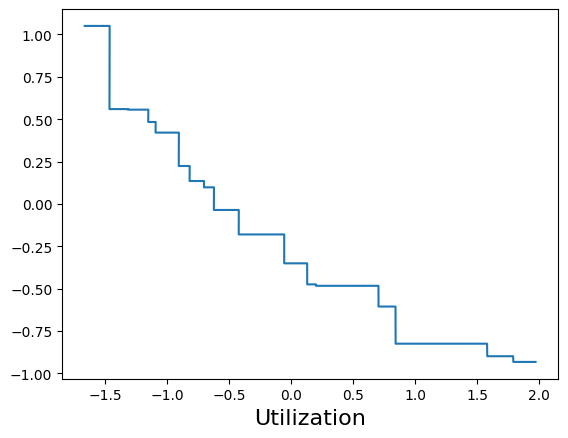}
    \caption{}
    \end{subfigure}
    \begin{subfigure}{0.4\textwidth}
    \includegraphics[width=\textwidth]{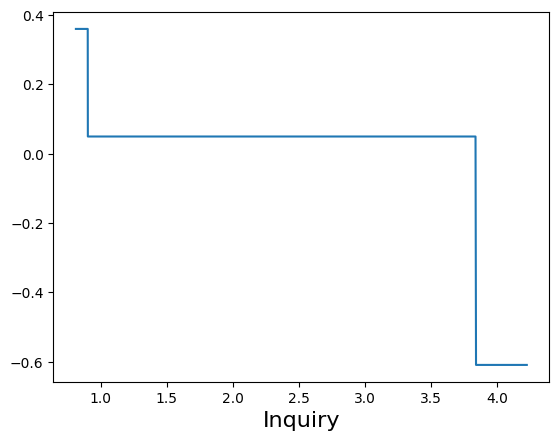}
    \caption{}
    \end{subfigure}
    \begin{subfigure}{0.4\textwidth}
    \includegraphics[width=\textwidth]{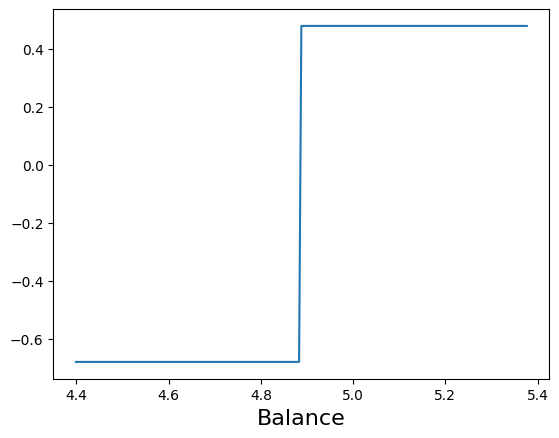}
    \caption{}
    \end{subfigure}
    \begin{subfigure}{0.4\textwidth}
    \includegraphics[width=\textwidth]{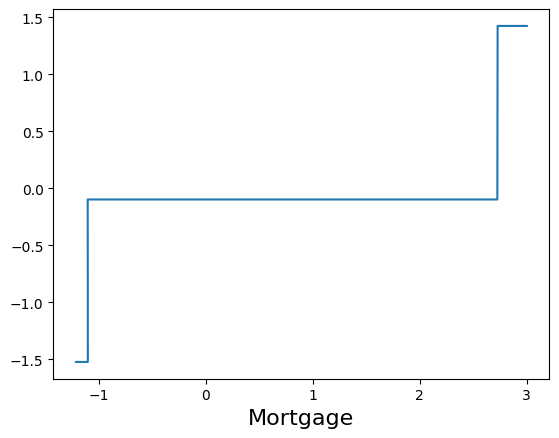}
    \caption{}
    \end{subfigure}
    \caption{ SimuCredit data: main effect plots of the most important features by LDA-XGB1 with FIV bound $\bar\epsilon=1.0$.}
    \label{fig:simuMonoFairME}
\end{figure}

\subsection{COMPAS}
For the Correctional Offender Management Profiling for Alternative Sanctions (COMPAS) dataset \citep{tempeh}, we attempt to predict whether a specific data point will re-commit a crime, i.e., predicting recidivism. The dataset contains the following ten features:
\begin{enumerate}\setlength\itemsep{0em}
    \item {\sf age}: how old the datapoint is
    \item {\sf juv\_fel\_count}: the number of felonies committed as a juvenile
    \item {\sf juv\_misd\_count}: the number of misdemeanours committed as a juvenile
    \item {\sf juv\_other\_count}: the number of other crimes committed as a juvenile
    \item {\sf priors\_count}: the number of crimes committed in the past
    \item {\sf age\_cat\_25 - 45}: a binary variable denoting within the age range $[25, 45]$
    \item {\sf age\_cat\_Greater than 45}: a binary variable denoting within the age range $(45, \infty)$
    \item {\sf age\_cat\_Less than 25}: a binary variable denoting within the age range $[0,25)$
    \item {\sf c\_charge\_degree\_F}: a binary variable denoting if the current charge is a felony
    \item {\sf c\_charge\_degree\_M}: a binary variable denoting if the current charge is a misdemeanour
\end{enumerate}
The {\sf gender} and {\sf race} info are not used for model development and they are saved in a separate file. The target is {\sf two\_year\_recid}, which is 1 if the data point committed another crime within two years of being released. 

Given the complexity of this dataset, we decrease the number of n\_estimators from 1000 to 100 so as to avoid overfitting. Other than that, all the XGBoost hyperparameters are set the same as in Section \ref{sec:simu}.

We again run three rounds of LDA-XGB1 algorithm with the following settings:
\begin{itemize}
\item[1)] {\sf Plain XGB1 Model.}
We applied Algorithm~\ref{alg:fairSequence} without fairness and monotone constraints, fitting a plain XGB1 model, using \emph{optbinning} to merge the bins. The feature importance and main effect plots are shown in Figures~\ref{fig:compasNoMonoFI} and \ref{fig:compasNoMonoME}, respectively. The main effect plots show apparent monotonic patterns for {\sf priors\_count}, and {\sf age}.
\item[2)] {\sf Monotonic XGB1 Model.} 
As observed from the unconstrained XGB1 model, we impose monotone constraints such that {\sf priors\_count}, {\sf juv\_fel\_count}, {\sf juv\_misd\_count}, and {\sf juv\_other\_count} must be increasing, while {\sf age} must be decreasing. 
Refit the LDA-XGB1 model and we have the feature importance and main effect plot results in Figures \ref{fig:compasMonoFI} and \ref{fig:compasMonoME}, respectively. Note that {\sf priors\_count} and {\sf age} both become far more important with the introduction of monotonic constraints, while {\sf juv\_fel\_count} drops enough to change places with {\sf age} and {\sf juv\_misd\_count} loses all importance completely. At the same time, {\sf priors\_count},  and {\sf age} now match our expectation based on reality.

\begin{figure}[htp!]
    \centering
    \includegraphics[width=0.6\linewidth]{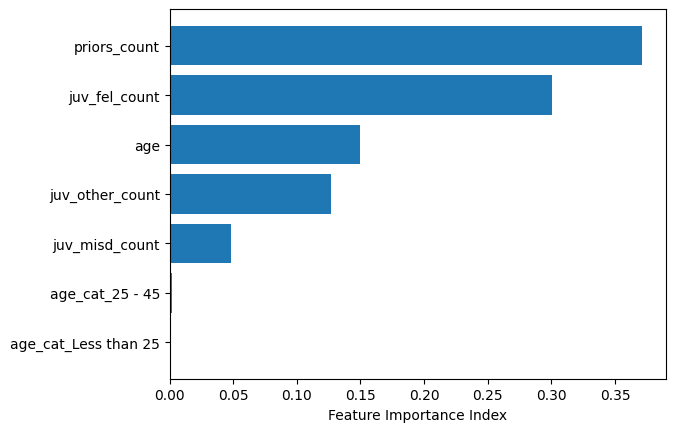}
    \caption{COMPAS data: feature importance plot of the plain XGB1 model. 
    }
    \label{fig:compasNoMonoFI}
\bigskip\bigskip\bigskip
    \centering
    \begin{subfigure}{0.32\textwidth}
    \includegraphics[width=\textwidth]{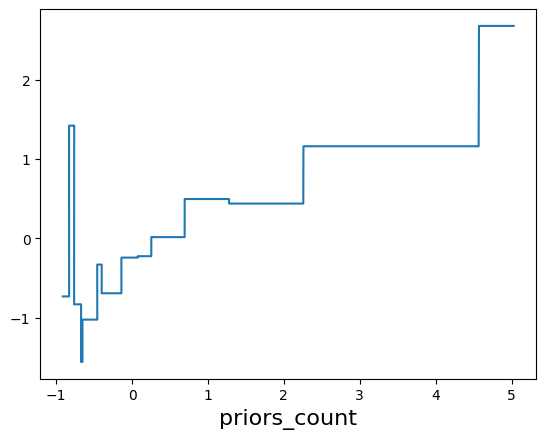}
    \caption{}
    \end{subfigure}
    \begin{subfigure}{0.32\textwidth}
    \includegraphics[width=\textwidth]{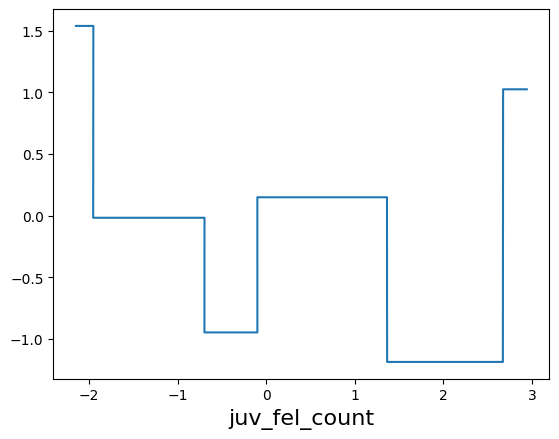}
    \caption{}
    \end{subfigure}
    \begin{subfigure}{0.32\textwidth}
    \includegraphics[width=\textwidth]{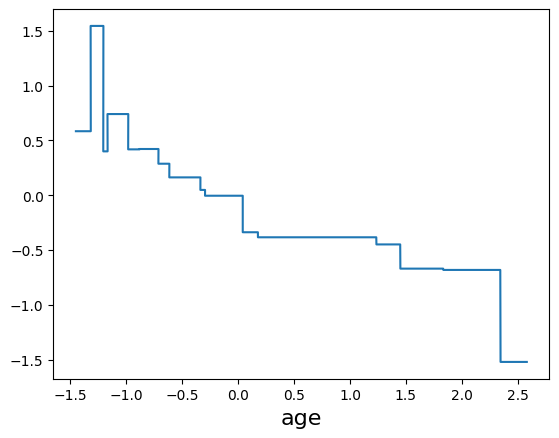}
    \caption{}
    \end{subfigure}
    \caption{COMPAS data: main effect plots of top features of the plain XGB1 Model. 
    }
    \label{fig:compasNoMonoME}
\end{figure}

\begin{figure}[htp!]
    \centering
    \includegraphics[width=0.6\linewidth]{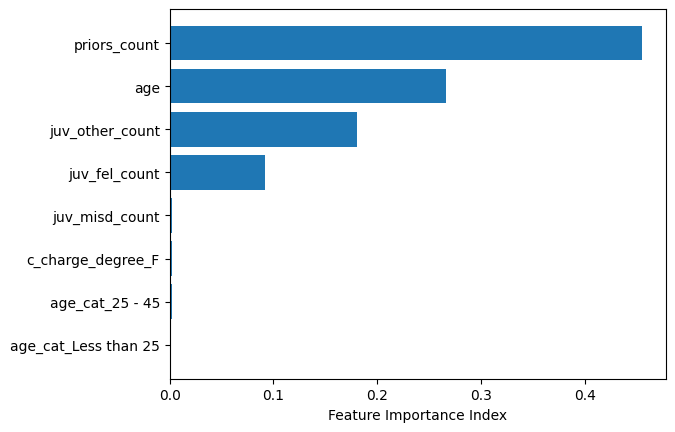}
    \caption{COMPAS data: feature importance plot of the monotonic XGB1 model. 
    }
    \label{fig:compasMonoFI}
\bigskip\bigskip\bigskip
    \centering
    \begin{subfigure}{0.32\textwidth}
    \includegraphics[width=\textwidth]{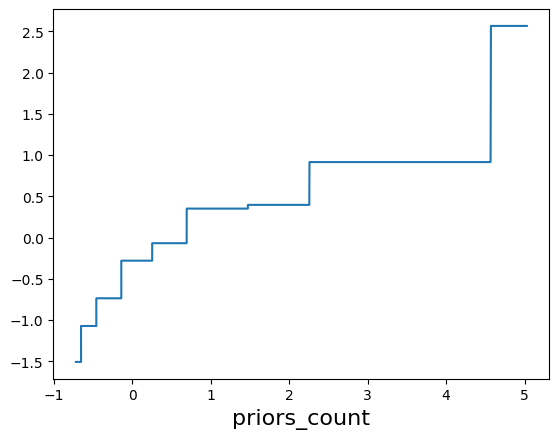}
    \caption{}
    \end{subfigure}
    \begin{subfigure}{0.32\textwidth}
    \includegraphics[width=\textwidth]{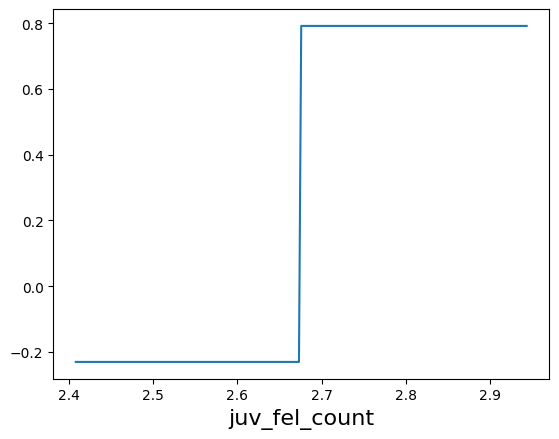}
    \caption{}
    \end{subfigure}
    \begin{subfigure}{0.32\textwidth}
    \includegraphics[width=\textwidth]{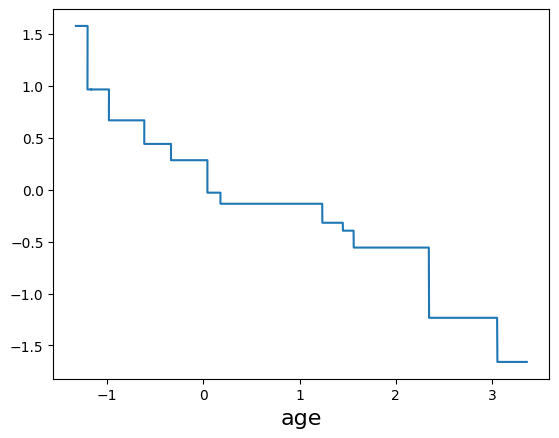}
    \caption{}
    \end{subfigure}
    \caption{COMPAS data: main effect plots of the most important features, {\sf priors\_count}, {\sf juv\_fel\_count}, and {\sf age}, of the monotonic XGB1 Model.}
    \label{fig:compasMonoME}
\end{figure}

\item[3)] {\sf Fair Monotonic XGB1 Model.}
Again consider the protected and reference `race' groups and define the fairness metric \eqref{fairObj}. We run the LGD-XGB1 algorithm with the fairness constraints and varying FIV bounds as shown in Figure~\ref{tab:compasPerf}. The trade-off between fairness and accuracy is shown in Figure~\ref{tab:compasPerf} for the testing performances, as well as the tabular results in Figure~\ref{tab:compasPerf} for both training and testing performances. It is also evident that as FIV bound decreases, the LDA-XGB1 model tends to fairer with higher AIR values, while the prediction accuracy gets sacrificed with lower AUC values. 
\end{itemize}

In our exploration of LDA-XGB1 models with fairness metrics, we observe significant bias in binary classification for the COMPAS dataset. The baseline models demonstrate low AIR values, indicating substantial fairness concerns. When fairness constraints with stringent FIV bounds are applied, the AIR improves, but it comes at the cost of significant reduction in model prediction accuracy. Figures~\ref{fig:compasMonoFairFI} and \ref{fig:compasMonoFairME} show a selected LDA-XGB1 model with FIV $\bar\epsilon = 0.14$ for model interpretation. In this model, where {\sf age} and {\sf priors\_count} emerge as the two most important features, with {\sf age} showing the decreasing trend and {\sf priors\_count} exhibiting an increasing trend.

\begin{figure}[htp!]
    \centering
    \begin{subfigure}{0.45\textwidth}
        \includegraphics[width=\linewidth]{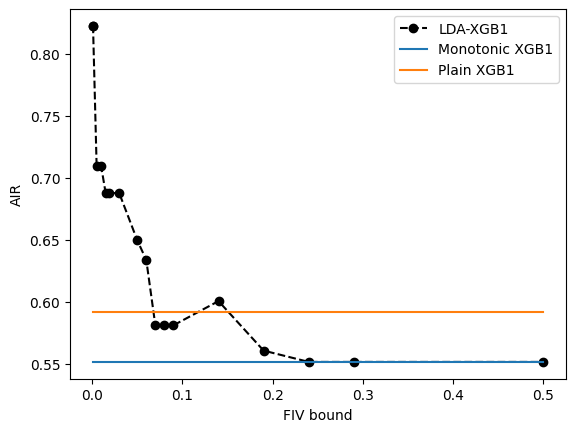}
        \caption{AIR}
        \label{fig:compasMonoAIR}
    \end{subfigure}
    \begin{subfigure}{0.45\textwidth}
        \includegraphics[width=\linewidth]{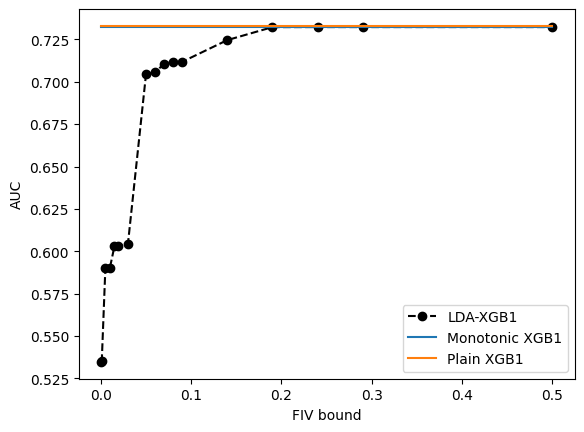}
        \caption{AUC}
        \label{fig:compasMonoFairAUC}
    \end{subfigure}
    \caption{COMPAS data: LDA-XGB1 with fairness and accuracy trade-off. 
    }
    \label{fig:compasPerf}

\bigskip\bigskip
    \centering
\begin{tabular}{r|rr|rr}\hline
LDA-XGB1 & AIR Train & AUC Train & AIR Test & AUC Test \\ \hline
0.5 & 0.571335 & 0.739177 & 0.553760 & 0.732118 \\
0.29 & 0.571335 & 0.739177 & 0.553760 & 0.732118 \\
0.24 & 0.571335 & 0.739177 & 0.553760 & 0.732118 \\
0.19 & 0.573842 & 0.738416 & 0.562920 & 0.731789 \\
0.14 & 0.625204 & 0.733396 & 0.598217 & 0.724442 \\
0.09 & 0.631568 & 0.724036 & 0.581194 & 0.711402 \\
0.08 & 0.631568 & 0.724036 & 0.581194 & 0.711402 \\
0.07 & 0.633951 & 0.721050 & 0.578653 & 0.710430 \\
0.06 & 0.680063 & 0.714979 & 0.633895 & 0.705063 \\
0.05 & 0.695318 & 0.713778 & 0.649832 & 0.703991 \\
0.03 & 0.797480 & 0.603075 & 0.687921 & 0.604531 \\
0.019 & 0.797480 & 0.601298 & 0.687921 & 0.602916 \\
0.015 & 0.797480 & 0.601299 & 0.687921 & 0.602908 \\
0.01 & 0.836524 & 0.585866 & 0.709498 & 0.590254 \\
0.005 & 0.836524 & 0.585866 & 0.709498 & 0.590254 \\\hline
Plain XGB1 & 0.573320 & 0.746819 & 0.589983 & 0.732669 \\
Monotonic XGB1 & 0.571335 & 0.739177 & 0.553760 & 0.732118 \\ \hline
\end{tabular}\caption{COMPAS data: AIR and AUC for the LDA-XGB1 models with varying FIV bounds $\bar\epsilon\in(0,0.5]$, together with baseline XGB1 models without using fairness constraints.}
\label{tab:compasPerf}
\end{figure}

\begin{figure}[htp!]
    \centering
    \includegraphics[width=0.6\linewidth]{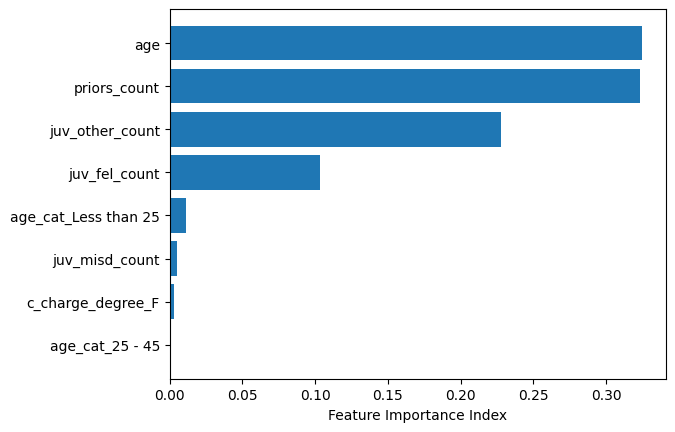}
    \caption{COMPAS data: feature importance of monotonic LDA-XGB1 with FIV bound $\bar\epsilon=0.14$. 
    }
    \label{fig:compasMonoFairFI}
\bigskip
    \centering
    \begin{subfigure}{0.32\textwidth}
    \includegraphics[width=\textwidth]{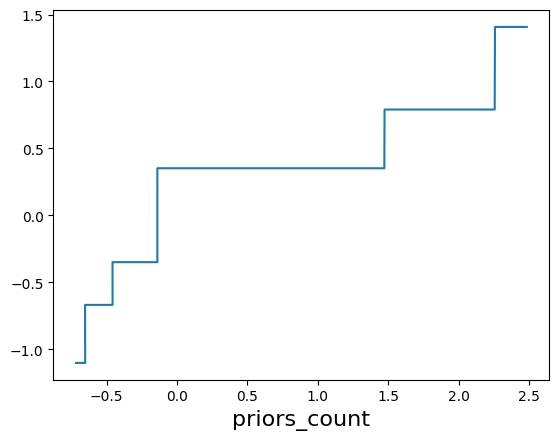}
    \caption{}
    \end{subfigure}
    \begin{subfigure}{0.32\textwidth}
    \includegraphics[width=\textwidth]{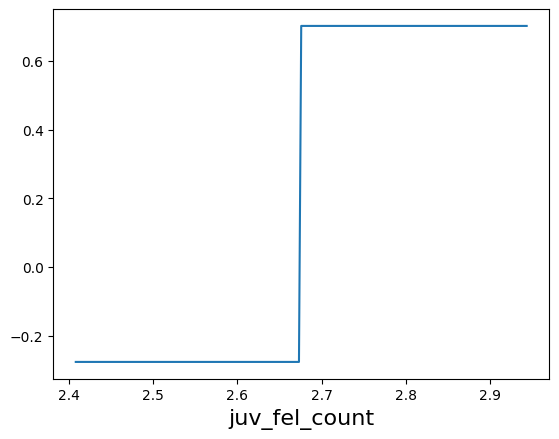}
    \caption{}
    \end{subfigure}
    \begin{subfigure}{0.32\textwidth}
    \includegraphics[width=\textwidth]{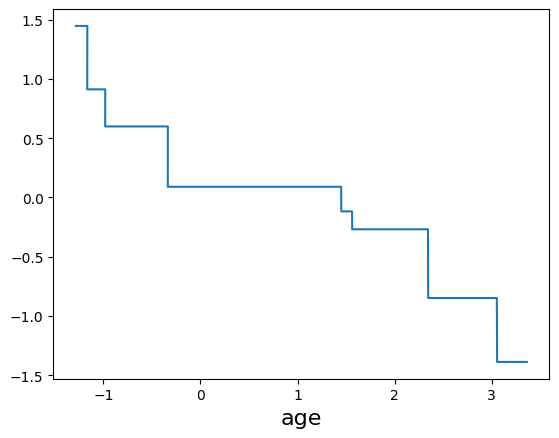}
    \caption{}
    \end{subfigure}
    \caption{COMPAS data: main effect plots of the most important features by LDA-XGB1 with FIV bound $\bar\epsilon=0.14$. 
    }
    \label{fig:compasMonoFairME}
\end{figure}

\section{Conclusion}
In this work, we introduced LDA-XGB1, a novel framework for developing less discriminatory and interpretable machine learning models for binary classification. By leveraging the predictive strength of XGBoost while integrating fairness constraints and monotonicity for interpretability, LDA-XGB1 achieves an effective balance between accuracy and fairness, particularly in high-stakes applications like financial lending and criminal justice.

Using both simulated (SimuCredit) and real-world (COMPAS) datasets, we demonstrated that LDA-XGB1 can mitigate biases against protected groups. Our biobjective optimization framework, which balances accuracy and fairness, allows for the flexible tuning of trade-offs, making LDA-XGB1 adaptable to various practical use cases.

\subsubsection*{Impact on the Financial Industry}
The adoption of LDA-XGB1 presents a promising pathway for financial institutions to comply with regulatory standards, such as the CFPB’s requirements for fair lending practices and model transparency. 
Recently, CFPB provided comments in response to the U.S. Department of the Treasury’s Request for Information on the use of artificial intelligence (AI) in the financial services sector \citep{CFPB2024}, The CFPB emphasized that AI adoption must comply with existing consumer protection laws, including the Equal Credit Opportunity Act (ECOA) and the Consumer Financial Protection Act (CFPA). These laws prohibit discrimination and ensure transparency in AI-driven decision-making processes, particularly in areas like lending and fraud detection.

The proposed LDA-XGB1 framework aligns well with CFPB’s emphasis on less discriminatory alternatives and the importance of model interpretability. 
As machine learning continues to transform decision-making processes in the financial sector, LDA-XGB1 can serve as a standard for incorporating ethical considerations into predictive modeling.

\subsubsection*{Closing Remarks}
In summary, LDA-XGB1 not only addresses current regulatory demands for fairness and interpretability but also sets the stage for future research on integrating fairness constraints into machine learning algorithms. As regulatory standards evolve and the demand for less discriminatory alternatives, the LDA-XGB1 framework can be expanded and refined to incorporate multiple fairness metrics, two-dimensional optbinning strategy for two-way interaction effects, and more complex use cases.

\section*{Acknowledgement}
This work is supported through a generous gift from the Wells Fargo Bank and partial funding from the Center for Trustworthy Artificial Intelligence through
Model Risk Management (TAIMing AI) at UNC Charlotte which is funded through the Division of Research, the School of Data Science, and the Klein College of Science.

\appendix
\section*{Appendix}\label{sec:app}
We present the result that \eqref{objILP} is a linear representation of the IV as a theorem, along with two detailed proofs thereof, one by induction, and one directly.
\begin{thm}
    The representation of the IV, \eqref{IV2}, which is nonlinear in $X$, is equivalent to the linear representation, \eqref{objILP}.
\end{thm}
\begin{proof}
    We proceed by induction. Suppose that there is a specific set (possibly infinite) of bins of pre-determined length. Let $k\in\N$ be some number of those bins, and denote the IV in \eqref{IV2} and proposed objective function in \eqref{objILP} for bins $1$ through $k$ as follows: \[IV_k(X):= \sum_{i=1}^k\left(\left(\sum_{j=1}^i(p_j-q_j)X_{ij}\right)\log\left(\dfrac{\sum_{j=1}^i p_jX_{ij}}{\sum_{j=1}^i q_iX_{ij}}\right)\right),\]
    \begin{align}
    f_k(X):=\sum_{i=1}^k \left(V_{ii}X_{ii} + \sum_{j=1}^{i-1}(V_{ij}-V_{ij+1})X_{ij}\right).\label{kObj}
    \end{align}
    Note that \[IV_{k+1}(X)=IV_k(X)+\left(\left(\sum_{j=1}^{k+1}(p_j-q_j)X_{{k+1}j}\right)\log\left(\dfrac{\sum_{j=1}^{k+1} p_jX_{{k+1}j}}{\sum_{j=1}^{k+1} q_jX_{{k+1}j}}\right)\right).\]
    Further, for a given row of matrix $[V_{ij}]$, let the IV over that row be denoted as \[V_{i\cdot}(X):=V_{ii}X_{ii} + \sum_{j=1}^{i-1}(V_{ij}-V_{ij+1})X_{ij}.\] Then \eqref{kObj} can be represented as \[f_k(X):=\sum_{i=1}^k V_{i\cdot}(X).\]

Consider the base case $n=2$. Then there are two cases: the bins are separate, that is, $X_{11}=X_{22}=1$, $X_{21}=0$; and the bins are merged, that is, $X_{11}=0$, $X_{21}=X_{22}=1$.

If the bins are separate, $X_{11}=X_{22}=1$, $X_{21}=0$, then \[IV_2(X)=(p_1-q_1)\log\left(\dfrac{p_1}{q_1}\right)+(p_2-q_2)\log\left(\dfrac{p_2}{q_2}\right)=V_{11}+V_{22}=f(X).\]

Similarly, if the bins are merged, $X_{11}=0$, $X_{21}=X_{22}=1$, then \[IV_2(X)=(p_1-q_1+p_2-q_2)\log\left(\dfrac{p_1+p_2}{q_1+q_2}\right)=V_{12}=V_{22}+V_{21}-V_{22}=f(X).\]

Suppose the induction step, that is, suppose $IV_{n-1}(X)=f_{n-1}(X)$. Since \[IV_{n}(X)=IV_{n-1}(X)+\left(\sum_{j=1}^n(p_j-q_j)X_{nj}\log\left(\dfrac{\sum_{j=1}^{n} p_jX_{{n}j}}{\sum_{j=1}^{n} q_{j}X_{{n}j}}\right)\right),\] we need only show that 
\begin{align}
    V_{n\cdot}(X)&=\sum_{j=1}^n(p_j-q_j)X_{nj}\log\left(\dfrac{\sum_{j=1}^{n} p_jX_{{n}j}}{\sum_{j=1}^{n} q_{j}X_{{n}j}}\right) \label{concEq}
\end{align}
Based on the constraint that a merged bin must be continuous, we have that, for each $i=1,\dots,n$, there exists $k\in\{1,\dots,n\}$ such that $X_{ij}=0$ for all $j< k$ and $X_{ij}=1$ for all $k \leq j\leq i$. Then \begin{align*}
    V_{i\cdot}(X)&=V_{ii}X_{ii} + \sum_{j=1}^{i-1}(V_{ij}-V_{ij+1})X_{ij}\\
    &=V_{ii}+\sum_{j=k}^{i-1}(V_{ij}-V_{ij+1})\\
    &=V_{ik}.
\end{align*}
Therefore $V_{n\cdot}=V_{nk}$. At the same time, since $X$ is fixed, we get that 
\begin{align*}
    \sum_{j=1}^n(p_j-q_j)X_{nj}\log\left(\dfrac{\sum_{j=1}^{n} p_jX_{{n}j}}{\sum_{j=1}^{n} q_{j}X_{{n}j}}\right)&=\sum_{j=k}^n(p_j-q_j)\log\left(\dfrac{\sum_{j=k}^{n} p_j}{\sum_{j=k}^{n} q_{j}}\right)\\
    &= V_{nk}\\
    &= V_{n\cdot}. 
\end{align*}
Consequently, $IV_n(X)=f_n(X)$, and the proof by induction is complete.
\end{proof}


\begin{thebibliography}{}
\bibitem[Agarwal et~al., 2021]{agarwal2021nam}
Agarwal, R., Melnick, L., Frosst, N., Zhang, X., Lengerich, B., Caruana, R. and Hinton, G. E. (2021). Neural additive models: Interpretable machine learning with neural nets. {\em Advances in Neural Information Processing Systems}, {\bf 34}, 4699--4711.

\bibitem[Bilodeau et~al., 2024]{bilodeau2024impossibility}
Bilodeau, B., Jaques, N., Koh, P. W. and Kim, B. (2024). Impossibility theorems for feature attribution. {\em Proceedings of the National Academy of Sciences}, 121(2), e2304406120.

\bibitem[Caton and Hass, 2024]{caton2024fairnss}
Caton, S. and Haas, C. (2024). Fairness in machine learning: A survey. {\em ACM Computing Surveys}, {\bf 56}(7), 1--38.

\bibitem[Chen and Guestrin, 2016]{chen2016xgboost}
Chen, T, and Guestrin, C. (2016). 
Xgboost: A scalable tree boosting system.
{\em Proceedings of the 22nd ACM SIGKDD international conference on knowledge discovery and data mining} (pp. 785--794). 

\bibitem[CFPB, 2022]{CFPB2022-03}
Consumer Financial Protection Bureau (2022). Consumer Financial Protection Circular 2022-03: Adverse action notification requirements in connection with credit decisions based on complex algorithms. May 26, 2022. 

\bibitem[Donini et~al., 2018]{donini2018fairnss}
Donini, M., Oneto, L., Ben-David, S., Shawe-Taylor, J. S. and Pontil, M. (2018). Empirical risk minimization under fairness constraints. {\em Advances in Neural Information Processing Systems}, 31.

\bibitem[Frotman and Meyer, 2024]{CFPB2024}
Frotman, S. and Meyer, E. (2024). CFPB Comment on Request for Information on Uses, Opportunities, and Risks of Artificial Intelligence in the Financial Services Sector. August 12, 2024.

\bibitem[Gillis et~al., 2024]{gillis2024lda}
Gillis, T. B., Meursault, V. and Ustun, B. (2024). Operationalizing the Search for Less Discriminatory Alternatives in Fair Lending. In {\em The 2024 ACM Conference on Fairness, Accountability, and Transparency} (pp. 377--387).

\bibitem[Haimes, 1971]{haimes1971mop}
Haimes, Y. (1971). On a bicriterion formulation of the problems of integrated system identification and system optimization. {\em IEEE transactions on systems, man, and cybernetics}, SMC-1 (3), 296--297.

\bibitem[Hastie and Tibshirani, 1990]{hastie1990}
Hastie, T. J. and Tibshirani, R. J. (1990). {\em Generalized Additive Models} (Vol. 43). CRC Press.

\bibitem[Hu et~al., 2023]{hu2023fanova}
Hu, L., Nair, V.~N., Sudjianto, A., Zhang, A. and Chen, J. (2023).
\newblock Interpretable machine learning based on functional ANOVA framework: algorithms and comparisons.
\newblock {\em arXiv preprint:2305.15670}.

\bibitem[Jeffreys, 1946]{jeffreys1946invariant}
Jeffreys, Harold (1946). An Invariant Form for the Prior Probability in Estimation Problems. {\em Proceedings of the Royal Society of London. Series A. Mathematical and Physical Sciences} (Vol. 186), 453--461.

\bibitem[Kullback and Liebler, 1951]{kullback1951information}
Kullback, Solomon and Leibler, Richard A. (1951). On Information and Sufficiency. {\em The Annals of Mathematical Statistics} (Vol. 22), 79--86.

\bibitem[Kumar et~al., 2020]{kumar2020prob}
Kumar, I. E., Venkatasubramanian, S., Scheidegger, C. and Friedler, S. (2020). Problems with Shapley-value-based explanations as feature importance measures. In {\em Proceedings of the 37th International Conference on Machine Learning}, PMLR, 119 (pp. 5491--5500).

\bibitem[Lundberg and Lee, 2017]{lundberg2017shap}
Lundberg, S.~M. and Lee, S.-I. (2017).
\newblock A unified approach to interpreting model predictions.
\newblock {\em Advances in Neural Information Processing Systems}, 30.

\bibitem[Lutz, 2020]{tempeh}
Lutz, Roman (2020). {TEst Machine learning PErformance exHaustively}. \url{https://pypi.org/project/tempeh/} Note: accessed February 9, 2024.


\bibitem[Mironchyk and Tchistiakov, 2017]{Mironchyk2017}
Mironchyk, Pavel and Tchistiakov, Victor (2017). Monotone Optimal Binning Algorithm for Credit Risk Modeling. DOI: 10.13140/RG.2.2.31885.44003

\bibitem[Navas-Palencia, 2020]{navas2020optimal}
Navas-Palencia, G. (2020). Optimal binning: mathematical programming formulation. 
{\em arXiv preprint:2001.08025}.

\bibitem[Ribeiro et~al., 2016]{ribeiro2016lime}
Ribeiro, M.~T., Singh, S., and Guestrin, C. (2016).
\newblock ``Why should i trust you?'' {E}xplaining the predictions of any
  classifier.
\newblock In {\em Proceedings of the 22nd ACM SIGKDD International Conference
  on Knowledge Discovery and Data Mining} (pp. 1135--1144).

\bibitem[Siddiqi, 2006]{siddiqi2006credit}
Siddiqi, Naeem (2006). {\em Credit {R}isk {S}corecards: {D}eveloping and {I}mplementing {I}ntelligent {C}redit {S}coring}. John Wiley \& Sons. 

\bibitem[Sudjianto and Zhang, 2021]{sudjianto2021design}
Sudjianto, A. and Zhang, A. (2021). 
Designing inherently interpretable machine learning models. 
{\em arXiv preprint:2111.01743}. 

\bibitem[Sudjianto et~al., 2023]{sudjianto2023piml}
Sudjianto, A., Zhang, A., Yang, Z., Su, Y., and Zeng, N. (2023). PiML toolbox for interpretable machine learning model development and diagnostics. 
{\em arXiv preprint:2305.04214.}

\bibitem[Vaughan et~al., 2018]{vaughan1018xnn}
Vaughan, J., Sudjianto, A., Brahimi, E., Chen, J. and Nair, V. N. (2018). Explainable neural networks based on additive index models. {\em arXiv preprint:1806.01933}.

\bibitem[Yang et~al., 2020]{yang2020exnn}
Yang, Z., Zhang, A. and Sudjianto, A. (2020).
\newblock Enhancing explainability of neural networks through architecture constraints.
\newblock {\em IEEE Transactions on Neural Networks and Learning Systems}, 32(6):2610--2621.


\bibitem[Yang et~al., 2021]{yang2021gami}
Yang, Z., Zhang, A., and Sudjianto, A. (2021).
\newblock {GAMI-Net}: An explainable neural network based on generalized additive models with structured interactions.
\newblock {\em Pattern Recognition}, 120:108192.

\bibitem[Zafar et~al., 2019]{zafar2019fairnss}
Zafar, M. B., Valera, I., Gomez-Rodriguez, M. and Gummadi, K. P. (2019). Fairness constraints: a flexible approach for fair classification. {\em Journal of Machine Learning Research}, 20(75), 1--42.
\end{thebibliography}
\end{document}